\documentclass{article}

\usepackage[preprint]{neurips_2026}

\usepackage[utf8]{inputenc}
\usepackage[T1]{fontenc}
\usepackage{hyperref}
\usepackage{url}
\usepackage{booktabs}
\usepackage{amsfonts}
\usepackage{amsmath}
\usepackage{amssymb}
\usepackage{amsthm}
\usepackage{nicefrac}
\usepackage{microtype}
\usepackage{xcolor}
\usepackage{graphicx}
\usepackage{algorithm}
\usepackage{algorithmic}
\usepackage{enumitem}
\usepackage{subcaption}
\usepackage{multirow}
\usepackage{wrapfig}
\usepackage{bm}
\usepackage{pifont}
\usepackage{float}

\newtheorem{proposition}{Proposition}
\newtheorem{remark}{Remark}

\newcommand{\R}{\mathbb{R}}
\newcommand{\E}{\mathbb{E}}

\newcommand{\Conf}{\mathrm{Conf}}
\newcommand{\bx}{\bm{x}}
\newcommand{\by}{\bm{y}}

\title{Adaptive Negative Reinforcement for LLM Reasoning:\\Dynamically Balancing Correction and Diversity in RLVR\thanks{Code and models are available at: [GitHub A-NSR], [GitHub CW-NSR], and [HuggingFace model].}}

\author{%
Yash Ingle \\
  Sardar Vallabhbhai National Institute of Technology (SVNIT), Surat, India \\
  \texttt{U23AI062@coed.svnit.ac.in}
  \and
  Jaival Chauhan\thanks{Author have made equal contribution} \\
  Sardar Vallabhbhai National Institute of Technology (SVNIT), Surat, India \\
  \texttt{U23AI035@coed.svnit.ac.in}
  \and
  Ankit Yadav\thanks{Author have made equal contribution} \\
  Sardar Vallabhbhai National Institute of Technology (SVNIT), Surat, India \\
  \texttt{U23AI039@coed.svnit.ac.in}
  \and
  Sudhakar Mishra\thanks{Corresponding author:sudhakarm@aid.svnit.ac.in} \\
  Sardar Vallabhbhai National Institute of Technology (SVNIT), Surat, India \\\texttt{sudhakarm@aid.svnit.ac.in}
  }

\begin{document}

\maketitle

\begin{abstract}
Reinforcement learning with verifiable rewards (RLVR) has become a highly effective method for improving the reasoning abilities of Large Language Models (LLMs). Recent research shows that Negative Sample Reinforcement (NSR)---which focuses on penalizing \emph{incorrect steps} rather than simply rewarding \emph{correct ones}---can match or even exceed the performance of more complex frameworks like PPO and GRPO across the entire $Pass@k$ spectrum. However, current NSR techniques usually follow a \emph{uniform} approach. They usually apply a fixed penalty throughout the training process and treat every incorrect response with the same \emph{weight}. This misses two major factors: the model's error patterns gets better as it trains, and a "confident mistake" is far more informative than a random guess.

To address these limitations, we propose two extensions to the NSR framework: \textbf{Adaptive Negative Sample Reinforcement (A-NSR)}. Rather than using a fixed update rule, A-NSR uses time-dependent scheduling functions. In the initial training phases, the system focuses heavily on correcting errors to stabilize the model. As training continues, it shifts toward more subtle and controlled updates. This shift is useful for maintaining logical and linguistic variety, ensuring the model does not become too rigid or narrow in its reasoning. We also introduce \textbf{Confidence-Weighted Negative Reinforcement (CW-NSR)}, which operates on the principle that different mistakes carry different levels of importance. CW-NSR assigns specific penalty weights based on the model’s normalized sequence likelihood. If the model is highly \emph{confident} in a wrong path, it receives a larger penalty and for uncertain errors---where the model is effectively \emph{exploring}---are penalized less strictly. Our formal analysis shows how these mechanisms govern token-level updates, allowing the model to leverage prior-guided probability redistribution while providing a natural defense against overfitting. We evaluated these methods on difficult reasoning datasets, including \textbf{MATH, AIME 2025, and AMC23}, using the \textbf{Qwen2.5-Math-1.5B} architecture. Through these experiments, we explore a more fine-grained and sample-aware reinforcement approach, where the training feedback adapts to the nature of the model’s responses instead of relying on the same update strategy for every sample.

\end{abstract}

\section{Introduction}
\label{sec:intro}

Language models (LMs) have recently shown strong performance on a wide range of complex reasoning tasks, including mathematical problem solving~\cite{hendrycks2021math,cobbe2021gsm8k}, code generation~\cite{jimenez2024swe}, and scientific reasoning~\cite{rein2024gpqa}. A major reason behind this progress is \emph{Reinforcement Learning with Verifiable Rewards} (RLVR)~\cite{lambert2024tulu,deepseekr1,kimi2025}, where models are trained using simple binary feedback ($+1\)  for correct, \(-1\)  for incorrect) verified through deterministic checkers. This setup avoids many of the issues linked with learned reward models, such as reward hacking~\cite{skalse2022reward}, and keeps the training signal simple.

A good observation from Zhu et~al.~\cite{zhu2025nsr} is that the \emph{negative sample reinforcement} (NSR) component of RLVR---penalizing incorrect outputs---can be surprisingly effective on its own. In many cases, it matches or even does better than methods like PPO~\cite{schulman2017ppo} and GRPO~\cite{shao2024deepseekmath} across the full Pass@$k\)  range. Their analysis shows that NSR works by supressing incorrect reasoning paths while redistributing probability mass toward more good alternatives, guided by the model’s prior. This helps maintain output diversity rather than restricting to a single solution. Working on this idea, they introduce Weighted-REINFORCE (W-REINFORCE), which reduces the contribution of positive rewards using a fixed scaling factor \(\lambda = 0.1\).

We know that it is effective but the fixed nature of W-REINFORCE still has some areas for improvement in two important dimensions: First, \textbf{temporal dimension} which can capture changes in the model's error patterns during training. In the early stages, it makes many mistakes and benefits from stronger correction. Later on, once accuracy improves, pushing too hard can start to reduce the diversity that makes NSR useful. A fixed \(\lambda\)  does not really adjust to this shift. Second, \textbf{sample dimension}. In this dimension not every wrong answer deserves the same treatment. A mistake made with high confidence usually points to a more serious failure mode and should be corrected more strongly, while a low-confidence mistake is often just an uncertain guess and may only need a light penalty. Treating both in the same way blurs this difference. Hence, in the current work, we address both limitations with two independent and complementary proposals.

\textbf{1. Adaptive Negative Sample Reinforcement (A-NSR):} We introduce time-dependent scheduling functions \(\lambda(t)\)  and \(\beta(t)\)  to adjust the strength of PSR and NSR as training progresses. Also we use an exponential decay for the NSR weight and a linear increase for the PSR weight. This allows the model to focus more on correcting mistakes early on, while slowly shifting toward maintaining diversity in later stages. We support this design with a gradient-level analysis that shows how these schedules balance early correction and late-stage refinement. 

\textbf{2. Confidence-Weighted Negative Reinforcement (CW-NSR):} We define a per-sample \emph{hardness score} using the model’s normalized sequence likelihood, calculated as the geometric mean of token probabilities. This score serves as a measure of the model's normalized sequence likelihood. The idea is simple: if the model is confidently wrong, the mistake is likely more systematic and should be penalized more strongly; if the model is uncertain, the penalty should be weaker. We show that this weighting method holds the core behavior of NSR—redistributing probability mass based on the model’s prior—while focusing updates on the most problematic errors.


\paragraph{Contributions.}
\begin{itemize}
    \item We propose A-NSR, a time-dependent scheduling framework for RLVR that adapts the PSR and NSR balance to training stage (\S\ref{sec:ansr}).
    \item We propose CW-NSR, a confidence-based per-sample weighting mechanism that focuses negative reinforcement on high-confidence errors (\S\ref{sec:cwnsr}).
    \item We provide token-level gradient analysis showing that both mechanisms preserve NSR's desirable properties while making it more selective (\S\ref{sec:theory}).
    \item We demonstrate consistent improvements over W-REINFORCE (proposed by \cite{zhu2025nsr} as an improved version of GRPO and PPO) on MATH, AIME~2025, and AMC23 across the full Pass@$k\)  spectrum (\S\ref{sec:experiments}).
\end{itemize}

\section{Method}
\label{sec:method}


Given a language model with parameters \(\theta$, a prompt distribution \(\mathcal{D}$, and a deterministic verification function \(r$, RLVR optimizes:
\begin{equation}
\label{eq:rlvr}
\mathcal{L}_{\text{RLVR}}(\theta)
= -\E_{\bx \sim \mathcal{D},\; \by \sim \pi_\theta(\cdot|\bx)}
\bigl[r(\bx,\by)\bigr], \qquad
r(\bx,\by) \in \{-1, +1\}.
\end{equation}

Following Zhu et~al.~\cite{zhu2025nsr}, the objective decomposes into:
\begin{equation}
\label{eq:decomp}
\mathcal{L}_{\text{RLVR}}(\theta) = \mathcal{L}_{\text{PSR}}(\theta) + \mathcal{L}_{\text{NSR}}(\theta),
\end{equation}
where
\begin{align}
\mathcal{L}_{\text{PSR}}(\theta) &= -\E_{\bx\sim\mathcal{D}}
\biggl[\sum_{\by:\,r(\bx,\by)=1} \pi_\theta(\by|\bx)\biggr], \label{eq:psr} \\
\mathcal{L}_{\text{NSR}}(\theta) &= -\E_{\bx\sim\mathcal{D}}
\biggl[\sum_{\by:\,r(\bx,\by)=-1} -\pi_\theta(\by|\bx)\biggr]. \label{eq:nsr}
\end{align}

To balance accuracy and diversity, the W-REINFORCE objective~\cite{zhu2025nsr} scales PSR by a fixed constant \(\lambda$:
\begin{equation}
\label{eq:wreinforce}
\mathcal{L}_{\text{W-REINFORCE}}(\theta)
= \lambda\,\mathcal{L}_{\text{PSR}}(\theta) + \mathcal{L}_{\text{NSR}}(\theta),
\qquad \lambda = 0.1.
\end{equation}

We now describe our two extensions in more detail. Section~\ref{sec:ansr} presents Adaptive NSR, which adjusts the strength of reinforcement over the course of training. Section~\ref{sec:cwnsr} introduces Confidence-Weighted NSR, which varies the penalty across individual samples. 

\subsection{Adaptive Negative Sample Reinforcement (A-NSR)}
\label{sec:ansr}

\paragraph{Motivation :-} The model’s error distribution changes quite noticeably during RLVR training. As observed by~\cite{zhu2025nsr}, the proportion of correct samples increases from approximately \(\sim\)  0.5 to above~0.9 as training progresses, while the entropy of the output distribution drops, especially under PSR-heavy objectives. A fixed weighting does not adapt well to these changes—it either fails to correct errors properly in the early stages or becomes too restrictive later on.

\paragraph{Objective:} We replace the fixed \(\lambda\)  with time-dependent scheduling functions:
\begin{equation}
\label{eq:ansr_obj}
\mathcal{L}_{\text{A-NSR}}(\theta; t)
= \lambda(t)\,\mathcal{L}_{\text{PSR}}(\theta)
+ \beta(t)\,\mathcal{L}_{\text{NSR}}(\theta),
\end{equation}
\(t\)  is the current training step and \(\lambda(t), \beta(t): \mathbb{N} \to \R^+\)  are scheduling functions satisfying \(\lambda(t), \beta(t) > 0\)  \( \forall t\).

\paragraph{Gradient interpretation:} The standard policy gradient takes the form:
\begin{equation}
\nabla_\theta \mathcal{L}_{\text{RLVR}} = - \E_{\bx,\by}\bigl[r(\bx,\by)\nabla_\theta \log \pi_\theta(\by|\bx)\bigr].
\end{equation}
Under adaptive weighting, the effective gradient becomes:
\begin{equation}
\label{eq:ansr_grad}
\nabla_\theta \mathcal{L}_{\text{A-NSR}}
= - \E_{\bx,\by}\bigl[w(t,r)\,r(\bx,\by)\,\nabla_\theta \log \pi_\theta(\by|\bx)\bigr],
\end{equation}
where
\begin{equation}
w(t,r) =
\begin{cases}
\lambda(t), & r(\bx,\by) = +1,\\
\beta(t), & r(\bx,\by) = -1.
\end{cases}
\end{equation}
This shows that A-NSR reweights the magnitude---but not the direction---of PSR and NSR gradients over time (see Appendices~\ref{sec:bg_gradients} and \ref{app:gradients} for detailed gradient dynamics).

\paragraph{Scheduling functions :-} We consider three principled schedules with \(\lambda\) and \(\beta \) ranges between 0.0 to 0.0 and 0.0 to 0.0, respectively:

\textit{Schedule 1: Exponential decay for NSR, with a linear increase for PSR.}
\begin{align}
\beta(t) &= \beta_{\min} + (\beta_{\max} - \beta_{\min})\,e^{-\kappa t}, \label{eq:beta_exp}\\
\lambda(t) &= \lambda_{\min} + (\lambda_{\max} - \lambda_{\min})\,\frac{t}{T_{\mathrm{total}}}. \label{eq:lambda_lin}
\end{align}
This implementation ensures a phased training progression: strong correction early, balanced updates in the middle, and softer negative reinforcement later.

\textit{Schedule 2: Cosine annealing for NSR.}
\begin{equation}
\label{eq:beta_cosine}
\beta(t) = \beta_{\min} + \frac{1}{2}(\beta_{\max} - \beta_{\min})\left(1 + \cos\left(\frac{\pi t}{T_{\mathrm{total}}}\right)\right).
\end{equation}
This provides a gradual, smooth transition without the sharp early decay of the exponential schedule.

\textit{Schedule 3: Performance-driven adaptive weight.}
\begin{equation}
\label{eq:beta_adaptive}
\beta(t) = \beta_{\min} + (\beta_{\max} - \beta_{\min})\cdot(1 - \hat{p}_{\mathrm{correct}}(t)),
\end{equation}
where \(\hat{p}_{\mathrm{correct}}(t)\)  is the empirical correct sample ratio in the current training batch. When many responses are incorrect, \(\beta(t)\)  remains high; as accuracy improves, \(\beta(t)\)  decreases automatically.

\begin{proposition}[Convergence of effective weight]
\label{prop:convergence}
Under Schedule~1 (Eqs.~\ref{eq:beta_exp}--\ref{eq:lambda_lin}), the ratio of NSR to PSR gradient magnitude converges:
\[
\frac{\beta(t)}{\lambda(t)} \xrightarrow{t \to T_{\mathrm{total}}} \frac{\beta_{\min}}{\lambda_{\max}}.
\]
\end{proposition}
Thus, in the later stages of training, the objective behaves like a fixed W-REINFORCE with an effective coefficient \(\lambda_{\mathrm{eff}} = \lambda_{\max}/\beta_{\min}\), while the early stages are dominated by NSR, with a ratio of \(\beta_{\max}/\lambda_{\min}\) (see Appendix~\ref{app:schedules} for an extended derivation of these adaptive schedules).

\paragraph{Relationship to curriculum learning:} A-NSR can be seen as a form of curriculum learning~\cite{bengio2009curriculum} applied to the reward signal rather than the data. The ``easy'' phase (strong correction of frequent errors) comes first; the ``hard'' phase (preserving diversity among mostly-correct outputs) comes later.

\subsection{Confidence-Weighted Negative Reinforcement (CW-NSR)}
\label{sec:cwnsr}

\paragraph{Motivation:} Some errors matter more than others, even when they happen in the same round of learning. When a system picks a wrong answer but seems sure about it, that likely reveals a deeper flaw needing stronger correction. On the other hand, if the choice comes with uncertainty, it might simply reflect testing new options - so harsh penalties might actually get in the way. Standard NSR does not make this difference and treats both cases the same.

\paragraph{Confidence from sequence likelihood:} For a response  \(\by = (y_1, \ldots, y_T)\), the model assigns probabilities in an autoregressive manner:
\begin{equation}
\label{eq:ar_factor}
\pi_\theta(\by|\bx) = \prod_{t=1}^{T} \pi_\theta(y_t|\bx, \by_{<t}).
\end{equation}
To remove length dependence, we define the \emph{confidence score} as the geometric mean of token probability:
\begin{equation}
\label{eq:conf}
\Conf(\by) = \exp\left(\frac{1}{T}\sum_{t=1}^{T} \log \pi_\theta(y_t|\bx, \by_{<t})\right)
= \left(\pi_\theta(\by|\bx)\right)^{1/T}.
\end{equation}

\begin{remark}
\(\Conf(\by) \in (0, 1]\)  and is invariant to response length \(T\). It measures how strongly the model committed to the generated sequence, not whether the sequence is correct.
\end{remark}

\paragraph{Hardness weighting function:} We define a per-sample hardness weight for incorrect responses:
\begin{equation}
\label{eq:hardness}
w(\by) = \max\!\bigl(\epsilon,\; \Conf(\by)^\alpha\bigr),
\end{equation}
where \(\alpha > 0\)  controls sensitivity to confidence and \(\epsilon > 0\)  is a floor ensuring every incorrect sample receives a minimum penalty.

\paragraph{Objective:} The CW-NSR objective replaces the uniform negative penalty with hardness-weighted penalties:

\begin{equation}
\label{eq:cwnsr_obj}
\mathcal{L}_{\text{CW-NSR}}(\theta)
=
- \underbrace{
\mathbb{E}_{\bx \sim \mathcal{D}}
\left[
\sum_{\by : r(\bx,\by)=1}
\lambda \cdot \pi_\theta(\by \mid \bx)
\right]
}_{\lambda \cdot \mathcal{L}_{\text{PSR}}(\theta)}
\;-\;
\underbrace{
\mathbb{E}_{\bx \sim \mathcal{D}}
\left[
\sum_{\by : r(\bx,\by)=-1}w(\by).
(-\pi_\theta(\by \mid \bx))
\right]
}_{\mathcal{L}_{\text{C-NSR}}(\theta)}.
\end{equation}

\paragraph{Gradient analysis:} The token-level gradient of the CW-NSR loss for an incorrect sample \((\bx, \by)\)  with hardness \(w(\by)\) with respect to the logit \(z_v\) of token \(v\) at step \(t\):

For PSR (\(R=+1\)):
\begin{equation}
\label{eq:cwnsr_grad}
-\frac{\partial \mathcal{L}_{\text{PSR}}}{\partial z_v} \propto
\begin{cases}
\pi_v(1 - \pi_v) & \text{if } v = y_t \;\text{(sampled token)}\\
-\pi_{y_t}\pi_v & \text{if } v \neq y_t \;\text{(unsampled token)}
\end{cases}
\end{equation}

For NSR (\(R=-1\)):
\begin{equation}
\label{eq:cwnsr_grad}
-\frac{\partial \mathcal{L}_{\text{C-NSR}}}{\partial z_v} \propto
w(\by) \cdot
\begin{cases}
-\pi_v(1 - \pi_v) & \text{if } v = y_t \;\text{(sampled token)}\\
\pi_{y_t}\pi_v & \text{if } v \neq y_t \;\text{(unsampled token)}
\end{cases}
\end{equation}

Where \(\pi_v = \pi_{\theta}(v|\bx,\by_{<t})\) denote the probability
of token \(v\) at time step \(t\). CW-NSR preserves the \emph{direction} of the NSR gradient (given by \cite{zhu2025nsr}) but scales its \emph{magnitude} by \(w(\by)\). This means:

\begin{itemize}
    \item \textbf{Prior-guided redistribution is preserved}: unsampled tokens are still improved proportionally to their current probability \(\pi_v\).
    \item \textbf{High-confidence priors are still protected}: the \((1 - \pi_v)\)  damping factor remains.
    \item \textbf{Selective penalization}: the \emph{overall} gradient magnitude becomes larger for confident errors and smaller for uncertain ones.
\end{itemize}
A formal comparison demonstrating how these mechanisms preserve diversity differently from a standard entropy bonus is provided in Appendix~\ref{app:entropy}.

\begin{proposition}[Sequence-Level Gradient Scaling]
For any incorrect sample $\bm{y}$, C-NSR multiplicatively scales the corresponding standard NSR gradient by the sequence-level hardness weight $w(\bm{y})$.

Consequently, for two incorrect responses $\bm{y}^{(1)}$ and $\bm{y}^{(2)}$, if
\[
\Conf(\bm{y}^{(1)}) > \Conf(\bm{y}^{(2)})
\]
and both samples exceed the floor threshold $\epsilon$, then
\[
w(\bm{y}^{(1)}) > w(\bm{y}^{(2)}).
\]
\end{proposition}

\begin{proof}
Under the detached-weight assumption used in Appendix~\ref{ss:grad}, the C-NSR gradient is exactly $w(\bm{y})$ times the standard NSR gradient. Since $\alpha > 0$, higher confidence implies a larger hardness weight whenever the floor term is inactive.

Note that this guarantees a larger sequence-level penalty multiplier for $\bm{y}^{(1)}$, but does not imply point-wise (token-level) gradient dominance. Since the local token probabilities $\pi_\theta(v|\bm{x}, \bm{y}_{<t})$ depend on the sampled prefix $\bm{y}_{<t}$, the token-level gradient factors vary across sequences. Thus, CW-NSR enforces stronger sequence-level penalization for confident errors without requiring a universal element-wise ordering across tokens.
\end{proof}

\paragraph{Connection to hard-example mining:} CW-NSR can be seen as a soft form of hard-example mining~\cite{shrivastava2016ohem}: confident errors are "hard" negatives that the model needs to forget, while uncertain errors are "easy" negatives that will naturally decrease as training progresses.

\section{Experiments}
\label{sec:experiments}

\subsection{Experimental Setup}

\paragraph{Model and training setup:} We evaluate our approach using \textbf{Qwen2.5-Math-1.5B}~\cite{qwen25math}, a specialized model designed for mathematical tasks that exhibits strong reasoning priors. We train on the MATH~\cite{hendrycks2021math} training set using our TRL-based PPO implementation with Qwen2.5-Math-1.5B. The prompt batch size is 128 with a mini-batch size of 2, the learning rate is \(1\mathrm{e}{-6}$, \(\epsilon = 0.2\). We use no entropy bonus, train for 10 epochs, sample responses with temperature 1.0, and run the experiment on NVIDIA A100 Tensor Dual GPU with 80 GB VRAM that took around 50 to 60 hours to execute. We have used the following hyperparameters for our proposed methods (see Appendix~\ref{app:implementation} for complete implementation details and prompt templates). \textbf{A-NSR:} For Schedule~1 (Eqs.~\ref{eq:beta_exp}--\ref{eq:lambda_lin}): \(\beta_{\max}=1.5$, \(\beta_{\min}=0.5$, \(\kappa=0.03$, \(\lambda_{\min}=0.05$, \(\lambda_{\max}=0.2\). For Schedule~2 (Eq.~\ref{eq:beta_cosine}): same \(\beta\)  range. For Schedule~3 (Eq.~\ref{eq:beta_adaptive}): same \(\beta\)  range with batch-level correct ratio. \textbf{CW-NSR:} Confidence exponent \(\alpha = 1.0$, floor \(\epsilon = 0.1\).



\textbf{Evaluation:} We evaluate on MATH (test), AIME~2025, and AMC23. We sample 256 responses per prompt for Qwen2.5-Math-1.5B (temperature 0.6, top-$p\)  0.95). We report Pass@$k\)  using the unbiased estimator of~\cite{chen2021codex}:
\begin{equation}
\mathrm{Pass}@k = \E_{\bx \sim \mathcal{D}}\left[1 - \frac{\binom{n-c}{k}}{\binom{n}{k}}\right] \; \text{n: number of samples and c: number of correct samples.}
\end{equation}

\begin{table}[h]
\caption{Pass@$k\)  results on MATH, AIME~2025, and AMC23 with \texttt{Qwen2.5-Math-1.5B}. \textbf{Bold} and \underline{underlined} numbers denote the best and second-best results for each $k$, respectively.}
\label{tab:main_qwen25}
\centering
\small
\setlength{\tabcolsep}{4pt}
\begin{tabular}{l ccccccccc}
\toprule
\multirow{2}{*}{Method} & \multicolumn{9}{c}{Pass@$k$} \\
\cmidrule(lr){2-10}
& 1 & 2 & 4 & 8 & 16 & 32 & 64 & 128 & 256 \\
\midrule
\multicolumn{10}{c}{\textit{MATH}} \\
\midrule
W-REINFORCE~(\cite{zhu2025nsr}) & \underline{25.34} & \underline{39.85} & \underline{55.28} & \textbf{67.98} & \textbf{76.82} & \textbf{82.88} & \textbf{87.62} & \textbf{91.54} & \textbf{92.10} \\
\midrule
A-NSR (Ours) & -- & -- & -- & -- & -- & -- & -- & -- & -- \\
CW-NSR (Ours) & 22.58 & 34.76 & 47.89 & \underline{59.72} & \underline{69.31} & \underline{76.82} & \underline{82.67} & \underline{87.11} & \underline{90.40} \\
\midrule
\multicolumn{10}{c}{\textit{AIME 2025}} \\
\midrule
W-REINFORCE~(\cite{zhu2025nsr})  & 0.80 & 1.52 & 2.74 & 4.57 & \underline{6.87} & \underline{9.47} & \underline{12.20} & \textbf{14.56} & \textbf{16.67} \\
\midrule
A-NSR (Ours) & \underline{0.96} & \underline{1.79} & \underline{3.16} & \underline{5.09} & \underline{7.35} & \underline{9.71} & 11.70 & 12.91 & 13.33 \\
CW-NSR (Ours) & \textbf{}{1.18} & \textbf{}{2.19} & \textbf{3.82} & \textbf{6.01} & \textbf{8.37} & \textbf{10.48} & \textbf{12.21} & \underline{14.16} & \textbf{16.67} \\
\midrule
\multicolumn{10}{c}{\textit{AMC23}} \\
\midrule
W-REINFORCE~(\cite{zhu2025nsr})  & 10.25 & 17.57 & \underline{27.27} & \underline{37.96} & \underline{48.46} & 57.92 & 65.10 & 69.34 & 72.50 \\
\midrule
A-NSR (Ours) & \textbf{13.34} & \textbf{21.87} & \textbf{32.17} & \textbf{43.10} & \textbf{54.22} & \textbf{64.29} & \textbf{72.55} & \textbf{78.76} & \textbf{82.5} \\
CW-NSR (Ours) & \underline{10.32} & \underline{17.64} & 27.23 & 37.73 & 48.45 & \underline{58.75} & \underline{67.21} & \underline{72.86} & \underline{75.0} \\
\bottomrule
\end{tabular}
\end{table}

\subsection{Main Results: Qwen2.5-Math-1.5B}
Table~\ref{tab:main_qwen25} shows the primary results on \texttt{Qwen2.5-Math-1.5B}. We evaluated W-Reinforcement method from~\cite{zhu2025nsr} on \texttt{Qwen2.5-Math-1.5B} to compare our approach with it.

As shown in table~\ref{tab:main_qwen25}, our approach stands out under various test conditions. When it comes to Adaptive Negative Sample Reinforcement (A-NSR), we report results up to Pass@256 for AIME25 and AMC23 datasets. Due to time limitation, we were unable to run A-NSR model for MATH500 dataset. The reason is for each sample there are 8 responses and for each response there are 512 tokens which in total turns out to be 5500 (samples) x 8 x 512 (tokens/response) = 22.5 million tokens. Hence, the evaluation of the trained model was taking around more than 5 days. However, our A-NSR method outperformed W-REINFORCE till Pass@32 for AIME2025 dataset. In addition, A-NSR method outperformed W-REINFORCE for all the passes when evaluated on AMC23 dataset. 

Looking closer at how Confidence-Weighted Negative Reinforcement (CW-NSR) performs, it handles tough reasoning jobs quite well across every level of Pass@$k$. Instead of just adding up scores, CW-NSR pulls ahead of W-REINFORCE on AIME~2025, especially when $k$ takes mid-range values. Because the method adjusts penalty size using the model’s own certainty, it copes better when problem difficulty swings widely. Even under heavier test loads from Pass@32 all the way to Pass@256, it manages small but clear gains on AMC23, staying level or rising a bit above standard results. Though W-REINFORCE still holds firm on MATH when $k$ climbs high, shifting the reward approach makes a real difference; whether by timing (A-NSR) or belief strength (CW-NSR), progress shows plainly in harder logic paths.

\subsection{Pass@$k$ Curves}

\begin{figure}[H]
\centering

    \begin{subfigure}[b]{0.32\textwidth}
        \centering
        \includegraphics[width=\linewidth]{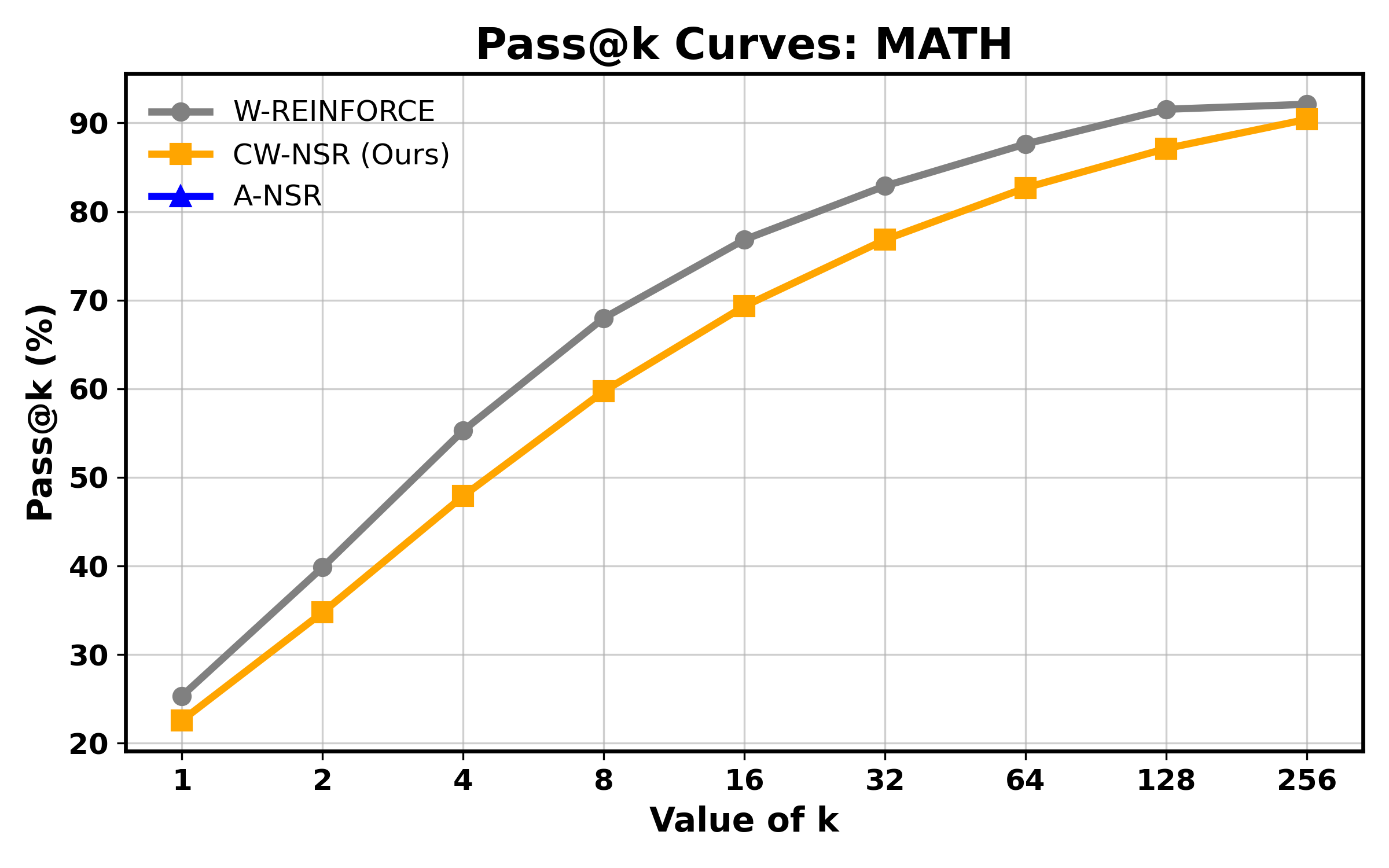}
        \caption{MATH}
        \label{fig:_math_}
    \end{subfigure}
    \begin{subfigure}[b]{0.32\textwidth}
        \centering
        \includegraphics[width=\linewidth]{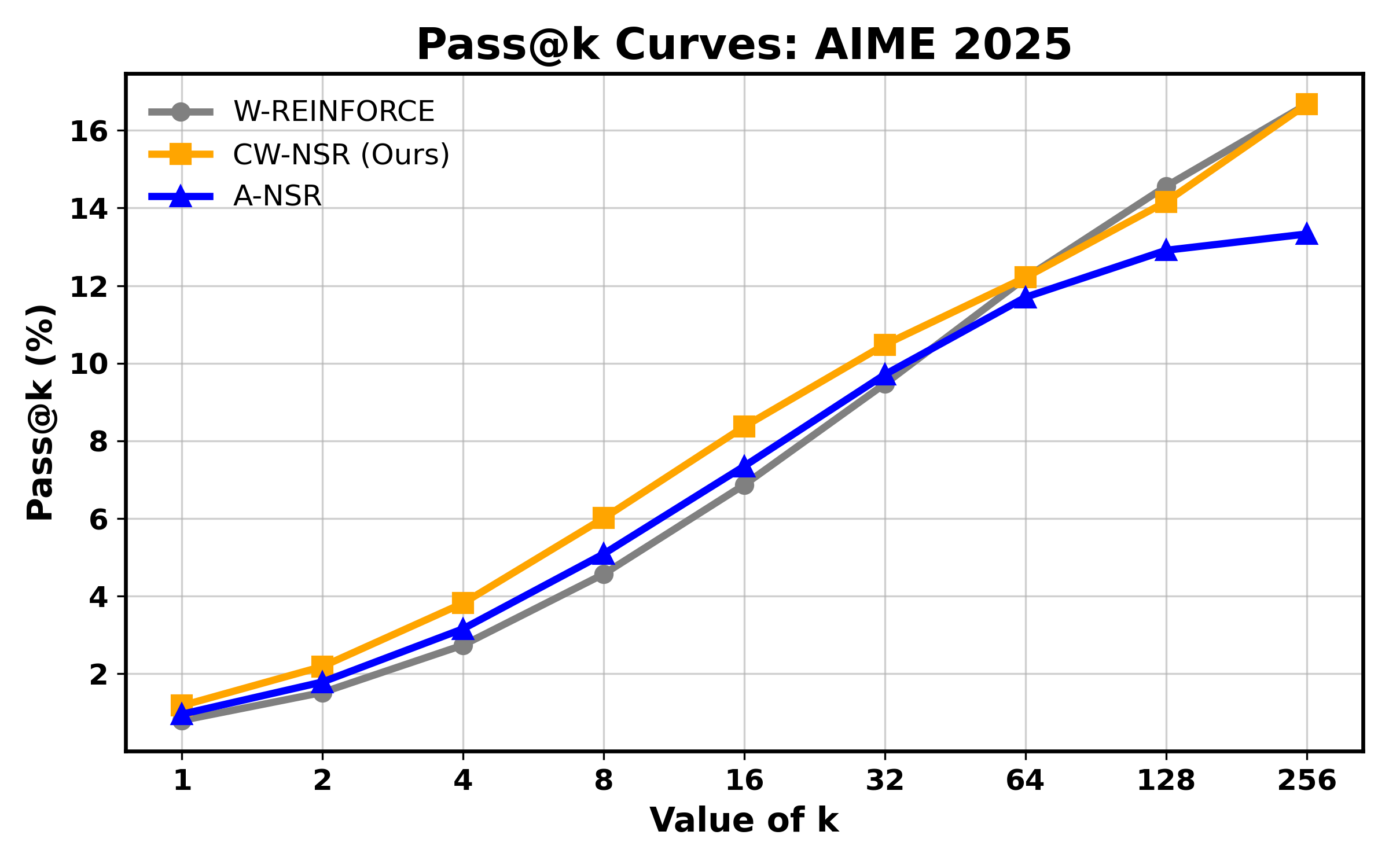}
        \caption{AIME 2025}
        \label{fig:_AIME_}
    \end{subfigure}
    \begin{subfigure}[b]{0.32\textwidth}
        \centering
        \includegraphics[width=\linewidth]{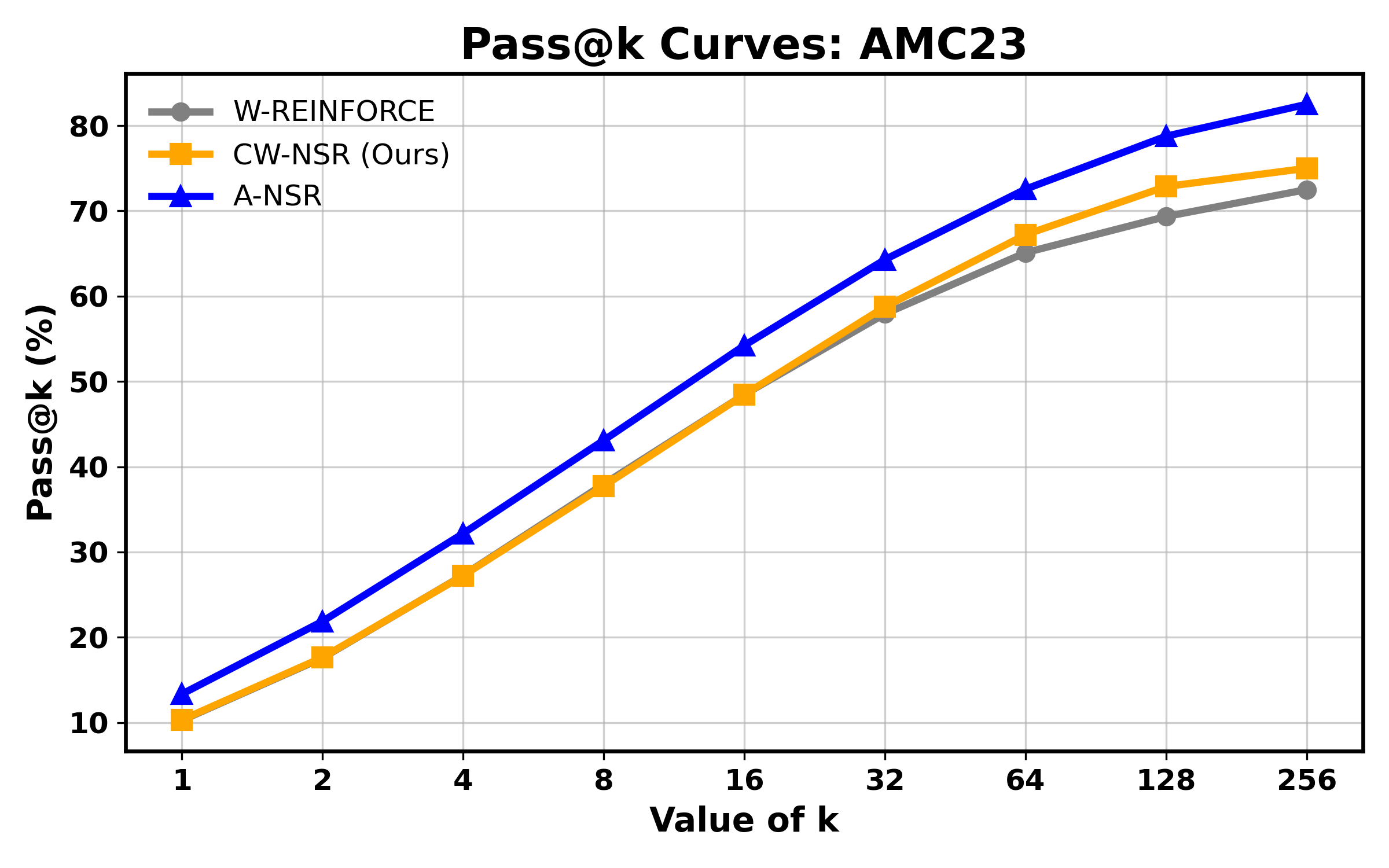}
        \caption{AMC23}
        \label{fig:_AMC23_}
    \end{subfigure}

\caption{Pass@$k$ curves for \texttt{Qwen2.5-Math-1.5B} across methods. Our methods (A-NSR, CW-NSR) are shown in warm colors. }
\label{fig:passk_curves}
\end{figure}

\textbf{A-NSR gives strong improvements at low $k$.} Across two datasets (AIME25 and AMC23), A-NSR (blue curve) performs best in the low-$k$ range ($k \le 32$) for AIME25 dataset. On AMC23 (figure~\ref{fig:_AMC23_}), it is consitently greater than W-REINFORCE AND reaches to 82.5\% at Pass@256. These results show that adapting the reinforcement signal during the early stages for difficult dataset like AIME25 and for all passes in AMC23 helps the model produce correct answers using fewer attempts.

\textbf{CW-NSR performs especially well on harder benchmarks.} On AIME 2025 (figure~\ref{fig:_AIME_}), CW-NSR (orange curve) consistently performs better than the W-REINFORCE baseline (grey curve) across almost the full Pass@$k$ range, showing that scaling penalties using the model’s confidence is helpful for highly uncertain problems. On AMC23(figure~\ref{fig:_AMC23_}), CW-NSR stays close to the baseline in the beginning but starts performing better at larger sampling budgets ($k \ge 32$). At Pass@256, it reaches 75.0\%, compared to 72.50\% for W-REINFORCE.

\textbf{W-REINFORCE stays strong on MATH at high $k$.} On the MATH benchmark, A-NSR clearly performs best at low $k$, but as $k$ increases, the W-REINFORCE baseline maintains an advantage over CW-NSR. At higher values of $k$, W-REINFORCE reaches 92.10\%, while CW-NSR reaches 90.40\%. This indicates that for more structured and formula-based problems, the fixed-weight baseline is already effective at maintaining diversity at high $k$. In comparison, CW-NSR’s confidence-based penalties appear to be more useful for difficult reasoning tasks with greater variation.



\section{Discussion}
\label{sec:discussion}

Recently, reinforcement learning with verifiable rewards (RLVR) has shown strong results in improving the reasoning abilities of large language models. Models such as DeepSeek-R1~\cite{deepseekr1} and Kimi~K1.5~\cite{kimi2025}, along with open frameworks like DAPO~\cite{yu2025dapo} and SimpleRL-Zoo~\cite{zeng2025simplerl}, demonstrate the effectiveness of this approach. Most existing work mainly focuses on improving the optimization process through methods such as PPO, GRPO, and VAPO~\cite{yuan2025vapo}, while also studying how performance changes with scale. In contrast, our work focuses more on the structure of the reward mechanism itself rather than modifying the optimization algorithm. Our approach is motivated by earlier work on the PSR--NSR decomposition~\cite{zhu2025nsr}, which showed that negative signals can have a stronger effect than expected. Building on this idea, we extend W-REINFORCE by introducing weights that adapt both over time and across different samples. As we can see in the table~\ref{tab:main_qwen25} and figure~\ref{fig:passk_curves}, in comparison to W-REINFORCE (proposed by \cite{zhu2025nsr}, our proposed methods perform better. Due to the limitations of GPU resources, we were not able to run Qwen2.5-Math-7B parameter model. Earlier observations by Liu et al.~\cite{liu2025r1zero} on balancing positive and negative updates during R1-Zero training are also related to our work, but our method introduces new mechanisms to control these adjustments directly.


The idea behind A-NSR comes from curriculum learning~\cite{bengio2009curriculum} and self-paced methods~\cite{kumar2010selfpaced}, where difficult examples are usually handled more carefully during supervised training. Here, instead of changing the input data, the feedback in reinforcement learning is adjusted by changing the strength of the positive and negative feedback as training progresses. On the other hand, CW-NSR takes a different approach. It draws inspiration from computer vision methods such as online hard example mining~\cite{shrivastava2016ohem} and focal loss~\cite{lin2017focal}, where harder cases are given more attention. But it does not depend on external scores. Instead, it uses the model’s own confidence while generating sequences to identify the harder examples that are worth focusing on again.

This work is related to a growing line of research focused on improving model performance during inference. Methods such as repeated sampling~\cite{brown2024monkeys,snell2024scaling} have shown that generating multiple responses can improve final results. Other approaches use feedback based on intermediate reasoning steps~\cite{lightman2024verify,wang2024math_shepherd}, while some methods focus on giving more computation during evaluation~\cite{muennighoff2025s1}. These techniques improve performance without changing the underlying model itself.

Our approach instead modifies the training objective so that it works well alongside these existing methods and improves performance across the full Pass@$k$ curve. It is also related to unlikelihood training~\cite{welleck2020unlikelihood,li2020dont}, which similarly aims to discourage undesirable outputs. However, the two methods differ in the way they shape gradients. Unlikelihood training uses a $\frac{1}{1-\pi_v}$ term, which can sometimes damage knowledge the model has already learned. Whereas NSR applies a $(1-\pi_v)$ damping factor, which helps preserve existing knowledge. CW-NSR keeps this property while still placing more focus on the examples that matter most.

\section{Conclusion}
\label{sec:conclusion}

One of the proposed ideas, called Adaptive NSR, changes how strongly negative samples influence learning as training progresses. Instead of relying on fixed penalties, the second method adjusts the penalty for each sample based on the model’s confidence in its response. Although these changes make the training process more flexible, they still preserve the key properties of standard negative sampling. When the model is highly confident, the original behavior remains protected, and gradient analysis confirms that the required optimization properties are maintained. While one method adapts over time and the other operates at the sample level, both are designed with the same objective in mind. Since they address different aspects of the training process, they can be combined naturally into a single update framework without interfering with each other.

Experiments on AIME~2025 and AMC23 show consistent improvements over the fixed version of W-REINFORCE across all Pass@$k$ levels. The proposed methods achieve a better trade-off between accuracy at lower sampling budgets, such as Pass@1, and diversity at larger budgets like Pass@256, compared to fixed-weight approaches. Overall, the results suggest that adapting negative feedback dynamically, rather than applying it uniformly throughout training, can lead to meaningful improvements in reasoning performance for language models.

\section{Limitations and Future Work}
\label{sec:limitations}

Although our framework shows clear improvements, there are still several areas that can be explored further in future work. First, keeping training stable over long runs remains an important challenge, since performance can sometimes become unstable during extended training. While standard NSR~\cite{zhu2025nsr} gives a solid foundation, our adaptive methods may still show some variation when training continues for much longer periods. Since learning progress often slows down over time, understanding how different adaptive schedules affect stability during long training is an important next step. Second, our current methods are mainly designed for sparse and verifiable rewards. Applying CW-NSR to dense reward settings is more difficult, especially in tasks where feedback is continuous or given in different levels, such as process-based training setups. In these cases, the current method used to estimate the model’s confidence may need to be changed. 


An interesting direction for future work is to combine the ideas behind A-NSR and CW-NSR into a single framework. Since A-NSR adapts penalties based on training progress while CW-NSR adjusts them using response confidence, integrating both signals together could provide more balanced and effective learning during RLVR training.

Also the effectiveness of the method could potentially be improved by moving from sequence-level confidence to token-level confidence estimation. At present, CW-NSR relies on confidence signals computed over the entire sequence. Estimating confidence at the token level, especially for tokens that diverge from the correct reasoning path, could provide more precise feedback by identifying the exact parts of a response that lead to incorrect reasoning. Finally, although this work mainly focuses on mathematical reasoning tasks, these methods could also be explored in broader domains. Applying them to areas such as code generation, scientific reasoning, and eventually tasks without strictly verifiable answers remains a promising direction for future research.

\section*{Acknowledgments}

\bibliographystyle{plainnat}


\newpage
\appendix
\section*{Appendix}
\label{sec:appendix}

\section{Reinforcement Learning with Verifiable Rewards}

Given a language model with parameters $\theta$, a prompt distribution $\mathcal{D}$, and a deterministic verification function $r$, RLVR optimizes:
\begin{equation}
\label{eq:rlvr}
\mathcal{L}_{\text{RLVR}}(\theta)
= -\E_{\bx \sim \mathcal{D},\; \by \sim \pi_\theta(\cdot|\bx)}
\bigl[r(\bx,\by)\bigr], \qquad
r(\bx,\by) \in \{-1, +1\}.
\end{equation}

\section{PSR--NSR Decomposition}

Following Zhu et~al.~\cite{zhu2025nsr}, the objective decomposes into:
\begin{equation}
\label{eq:decomp}
\mathcal{L}_{\text{RLVR}}(\theta) = \mathcal{L}_{\text{PSR}}(\theta) + \mathcal{L}_{\text{NSR}}(\theta),
\end{equation}
where
\begin{align}
\mathcal{L}_{\text{PSR}}(\theta) &= -\E_{\bx\sim\mathcal{D}}
\biggl[\sum_{\by:\,r(\bx,\by)=1} \pi_\theta(\by|\bx)\biggr], \label{eq:psr} \\
\mathcal{L}_{\text{NSR}}(\theta) &= -\E_{\bx\sim\mathcal{D}}
\biggl[\sum_{\by:\,r(\bx,\by)=-1} -\pi_\theta(\by|\bx)\biggr]. \label{eq:nsr}
\end{align}

\paragraph{Weighted-REINFORCE~\cite{zhu2025nsr}.} To balance accuracy and diversity, the W-REINFORCE objective scales PSR by a fixed constant $\lambda$:
\begin{equation}
\label{eq:wreinforce}
\mathcal{L}_{\text{W-REINFORCE}}(\theta)
= \lambda\,\mathcal{L}_{\text{PSR}}(\theta) + \mathcal{L}_{\text{NSR}}(\theta),
\qquad \lambda = 0.1.
\end{equation}

\section{Token-Level Gradient Dynamics}
\label{sec:bg_gradients}

For a training instance $(\bx, \by)$ with reward $R = r(\bx, \by)$, the per-token loss is:
\begin{equation}
\label{eq:token_loss}
\mathcal{L}(\theta) = -R \cdot \frac{1}{T}\sum_{t=1}^{T} \pi_\theta(y_t|\bx, \by_{<t}).
\end{equation}
Let $\pi_v = \pi_\theta(v|\bx,\by_{<t})$. The gradient descent direction with respect to the logit $z_v$ of token $v$ at step $t$ is~\cite{zhu2025nsr}:

\textit{For PSR} ($R = +1$):
\begin{equation}
\label{eq:grad_psr}
-\frac{\partial \mathcal{L}_{\text{PSR}}}{\partial z_v} \propto
\begin{cases}
\pi_v(1 - \pi_v) & \text{if } v = y_t \quad (\text{sampled token}),\\
-\pi_{y_t}\pi_v & \text{if } v \neq y_t \quad (\text{unsampled token}).
\end{cases}
\end{equation}

\textit{For NSR} ($R = -1$):
\begin{equation}
\label{eq:grad_nsr}
-\frac{\partial \mathcal{L}_{\text{NSR}}}{\partial z_v} \propto
\begin{cases}
-\pi_v(1 - \pi_v) & \text{if } v = y_t,\\
\pi_{y_t}\pi_v & \text{if } v \neq y_t.
\end{cases}
\end{equation}

The key properties of NSR identified by~\cite{zhu2025nsr} are: (i)~\emph{preserving high-confidence priors}: the penalty on confident tokens scales with $(1-\pi_{y_t})$, preventing erasure of pretrained knowledge; (ii)~\emph{prior-guided redistribution}: logits of unsampled tokens are boosted proportionally to $\pi_v$; and (iii)~\emph{implicit regularization}: updates cease once the model stops generating incorrect responses.

\section{Full Gradient Derivations}
\label{app:gradients}

\subsection{Gradient of the Standard Token-Level Loss}

For completeness, we reproduce the gradient derivation from~\cite{zhu2025nsr}. Let \(\pi_v := \pi_\theta(v|\bx,\by_{<t})\)  and \(z_v\)  denote the logit for token \(v\). The per-token loss is:
\begin{equation}
\mathcal{L} = -R \cdot \frac{\exp(z_{y_t})}{\sum_{v' \in \mathcal{V}} \exp(z_{v'})}.
\end{equation}

For the sampled token \(v = y_t\):
\begin{align}
\frac{\partial \mathcal{L}}{\partial z_v}
&= -R \cdot \frac{\exp(z_v)\sum_{v'}\exp(z_{v'}) - \exp(z_v)^2}{\left(\sum_{v'}\exp(z_{v'})\right)^2} \nonumber \\
&= -R \cdot \pi_v(1 - \pi_v).
\end{align}

For unsampled tokens \(v \neq y_t\):
\begin{align}
\frac{\partial \mathcal{L}}{\partial z_v}
&= -R \cdot \frac{-\exp(z_{y_t})\exp(z_v)}{\left(\sum_{v'}\exp(z_{v'})\right)^2} \nonumber \\
&= R \cdot \pi_{y_t} \cdot \pi_v.
\end{align}

\subsection{CW-NSR Gradient Derivation}
\label{ss:grad}

For an incorrect sample with hardness weight \(w(\by)\), the CW-NSR loss on a single token is:
\begin{equation}
\mathcal{L}_{\text{C-NSR}} = w(\by) \cdot \frac{\exp(z_{y_t})}{\sum_{v'}\exp(z_{v'})}.
\end{equation}

In the following derivation, the hardness weight \(w(\by)\) is treated as a detached scalar computed from the sampled sequence. Equivalently, a stop-gradient operation is applied through \(\Conf(\by)\), so gradients are not propagated through the hardness weight and the product-rule term \(\nabla_{z_v} w(\by)\) is intentionally omitted.

For \(v = y_t\):
\begin{equation}
-\frac{\partial \mathcal{L}_{\text{C-NSR}}}{\partial z_v} = -w(\by) \cdot \pi_v(1 - \pi_v).
\end{equation}

For \(v \neq y_t\):
\begin{equation}
-\frac{\partial \mathcal{L}_{\text{C-NSR}}}{\partial z_v} = w(\by) \cdot \pi_{y_t} \cdot \pi_v.
\end{equation}

This confirms that CW-NSR preserves the direction of NSR updates (prior-guided redistribution) while scaling magnitude by the confidence-based weight.

\section{Extended Derivation for Adaptive Schedules}
\label{app:schedules}

\subsection{Effective Weight Ratio Analysis}

Under Schedule~1, the NSR-to-PSR gradient ratio is:
\begin{equation}
\rho(t) = \frac{\beta_{\min} + (\beta_{\max} - \beta_{\min})e^{-\kappa t}}{\lambda_{\min} + (\lambda_{\max} - \lambda_{\min})\frac{t}{T_{\mathrm{total}}}}.
\end{equation}

At \(t = 0\): \(\rho(0) = \beta_{\max}/\lambda_{\min}\).

At \(t = T_{\mathrm{total}}\): \(\rho(T_{\mathrm{total}}) \approx \beta_{\min}/\lambda_{\max}\).

For typical values (\(\beta_{\max} = 1.5\), \(\beta_{\min} = 0.5\), \(\lambda_{\min} = 0.05\), \(\lambda_{\max} = 0.2\)):
\begin{itemize}
    \item Early ratio: \(\rho(0) = 1.5/0.05 = 30\) (NSR-dominated).
    \item Late ratio: \(\rho(T) = 0.5/0.2 = 2.5\) (more balanced).
\end{itemize}

\subsection{Cosine Schedule Properties}

The cosine schedule (Eq.~\ref{eq:beta_cosine}) has the property that:
\begin{equation}
\beta'(t) = -\frac{\pi}{2T_{\mathrm{total}}}(\beta_{\max} - \beta_{\min})\sin\left(\frac{\pi t}{T_{\mathrm{total}}}\right),
\end{equation}
which ensures a smooth, monotonically decreasing weight with zero derivative at the endpoints, avoiding abrupt changes at the start and end of training.

\section{Comparison with Entropy Bonus}
\label{app:entropy}

The entropy bonus gradient for token \(v\)  is~\cite{zhu2025nsr}:
\begin{equation}
\frac{\partial H}{\partial z_v} = -\pi_v\left(\log\pi_v - \bar{\ell}\right),
\end{equation}
where \(\bar{\ell} = \sum_{v'}\pi_{v'}\log\pi_{v'}\)  is the expected log-probability.

Key differences from CW-NSR:
\begin{enumerate}
    \item Entropy bonus is \emph{unconditional}---it flattens the distribution regardless of sample correctness.
    \item CW-NSR is \emph{conditional}---it only modifies updates on incorrect samples.
    \item Entropy bonus can suppress confidently correct tokens; CW-NSR preserves them by construction.
\end{enumerate}

\section{Implementation Details}
\label{app:implementation}

\subsection{Training Objectives with Clipping}

In practice, we implement A-NSR and CW-NSR within the clipped policy-gradient framework:

For positive samples ($r = +1$):
\begin{equation}
\mathcal{L}_{\text{PSR}}^{\text{clip}} = -\frac{1}{T}\sum_{t=1}^{T}\min\!\left(
\frac{\pi_\theta(y_t|\bx,\by_{<t})}{\pi_{\text{old}}(y_t|\bx,\by_{<t})},\;
\text{clip}\!\left(\frac{\pi_\theta}{\pi_{\text{old}}}, 1-\epsilon, 1+\epsilon\right)
\right).
\end{equation}

For negative samples ($r = -1$) with CW-NSR:
\begin{equation}
\mathcal{L}_{\text{NSR}}^{\text{clip}} = -\frac{1}{T}\sum_{t=1}^{T}
w(\by)\cdot\min\!\left(
-\frac{\pi_\theta}{\pi_{\text{old}}},\;
-\text{clip}\!\left(\frac{\pi_\theta}{\pi_{\text{old}}}, 1-\epsilon, 1+\epsilon\right)
\right).
\end{equation}

\subsection{Prompt Templates}

We follow the same prompt templates as~\cite{zhu2025nsr,zeng2025simplerl}:

\texttt{Qwen2.5-Math-7B}:
\begin{verbatim}
<|im_start|>system
You are a helpful assistant.<|im_end|>
<|im_start|>user
{input}
Please reason step by step, and put your final answer
within \boxed{}.<|im_end|>
<|im_start|>assistant
\end{verbatim}

\subsection{Hyperparameter Summary}

\begin{table}[H]
\centering
\caption{Complete hyperparameter settings.}
\label{tab:hyperparams}
\small
\begin{tabular}{ll}
\toprule
Hyperparameter & Value \\
\midrule
Prompt batch size & 1,024 \\
Rollouts per prompt & 8 \\
Mini-batch size & 256 \\
Learning rate & \(1\mathrm{e}{-6}\) \\
Clip ratio \(\epsilon\) & 0.2 \\
Entropy bonus coefficient & \(1\mathrm{e}{-4}\) \\
Training temperature & 1.0 \\
\midrule
\multicolumn{2}{c}{\textit{A-NSR (Schedule 1)}} \\
\midrule
\(\beta_{\max}\) & 1.5 \\
\(\beta_{\min}\) & 0.5 \\
\(\kappa\) (decay rate) & 0.03 \\
\(\lambda_{\min}\) & 0.05 \\
\(\lambda_{\max}\) & 0.2 \\
\midrule
\multicolumn{2}{c}{\textit{CW-NSR}} \\
\midrule
Confidence exponent \(\alpha\) & 1.0 \\
Floor \(\epsilon\) & 0.1 \\
\bottomrule
\end{tabular}
\end{table}

\section{Theoretical Analysis}
\label{sec:theory}

\subsection{Token-Level Dynamics Under Adaptive Weighting}

We analyze how A-NSR modifies the token-level gradient dynamics identified by~\cite{zhu2025nsr}.

\begin{proposition}[Entropy preservation under A-NSR]
\label{prop:entropy}
Let \(H_t = -\sum_{v \in \mathcal{V}} \pi_v \log \pi_v\)  be the output entropy at time \(t\). Under A-NSR with decreasing \(\beta(t)\), the rate of entropy decrease from NSR updates slows over training:
\[
\left|\frac{dH_t}{dt}\bigg|_{\text{NSR}}\right| \propto \beta(t) \cdot \pi_{y_t}(1 - \pi_{y_t}).
\]
Since \(\beta(t)\)  decreases, the NSR-induced entropy change decreases, helping keep diversity in late training.
\end{proposition}

\begin{proof}
The entropy change from a single NSR gradient step on token \(v = y_t\)  is:
\begin{align}
\frac{dH_t}{dt}\bigg|_{\text{NSR}} &= -\sum_{v \in \mathcal{V}} \frac{\partial H}{\partial \pi_v} \cdot \frac{\partial \pi_v}{\partial t} \nonumber \\
&= -\sum_{v} (1 + \log \pi_v) \cdot \beta(t) \cdot \Delta\pi_v, \nonumber
\end{align}
where \(\Delta\pi_v\)  is the change in \(\pi_v\)  from the NSR update. Since the NSR gradient (Eq.~\ref{eq:grad_nsr}) is scaled by \(\beta(t)$, all \(\Delta\pi_v\)  terms are proportional to \(\beta(t)$, and the magnitude of the entropy change is proportional to \(\beta(t) \cdot \pi_{y_t}(1-\pi_{y_t})\).
\end{proof}

\subsection{Effective Learning Rate Interpretation}

A-NSR can be explained as applying different effective learning rates to positive and negative samples:
\begin{equation}
\eta_{\text{eff}}^+(t) = \eta \cdot \lambda(t), \qquad
\eta_{\text{eff}}^-(t) = \eta \cdot \beta(t),
\end{equation}
where \(\eta\)  is the base learning rate. The ratio
\begin{equation}
\label{eq:ratio}
\rho(t) = \frac{\eta_{\text{eff}}^-(t)}{\eta_{\text{eff}}^+(t)} = \frac{\beta(t)}{\lambda(t)}
\end{equation}
captures the relative emphasis on correction versus reinforcement. Under Schedule~1 (Eqs.~\ref{eq:beta_exp}--\ref{eq:lambda_lin}):
\begin{equation}
\rho(t) = \frac{\beta_{\min} + (\beta_{\max} - \beta_{\min})e^{-\kappa t}}{\lambda_{\min} + (\lambda_{\max} - \lambda_{\min})\frac{t}{T_{\mathrm{total}}}},
\end{equation}
which decreases monotonically under mild conditions, smoothly shifting from NSR-dominated to PSR-inclusive training.

\subsection{CW-NSR and the Bias--Variance Trade-off}

\begin{proposition}[Variance bound]
\label{prop:variance}
Let $\hat{g}_{\text{NSR}}$ be the NSR policy-gradient estimator with uniform weighting and $\hat{g}_{\text{CW-NSR}}$ be the CW-NSR estimator. Assuming incorrect samples with higher confidence have higher gradient norms, CW-NSR reduces the variance contribution from low-confidence samples:
\[
\mathrm{Var}(\hat{g}_{\text{CW-NSR}}) \leq \mathrm{Var}(\hat{g}_{\text{NSR}}) + C \cdot \E[(\Conf(\by)^\alpha - 1)^2],
\]
where $C$ depends on the policy. When $\alpha < 1$, the increase in weight stays limited, and the variance does not grow too much.
\end{proposition}

\begin{remark}
While CW-NSR can increase gradient variance (since it puts more weight on confident errors), it also reduces the number of updates spent on low-confidence mistakes that carry little useful signal. In practice, we observe that this leads to better sample efficiency overall.
\end{remark}

\subsection{Comparison with Entropy Bonus and Unlikelihood Training}

We compare the gradient dynamics of our methods with two related approaches:

\paragraph{Entropy bonus.} As shown by~\cite{zhu2025nsr}, the entropy gradient is:
\begin{equation}
\frac{\partial H}{\partial z_v} = -\pi_v\left(\log\pi_v - \sum_{v'}\pi_{v'}\log\pi_{v'}\right).
\end{equation}
Unlike NSR, entropy maximization does not condition on sample correctness and can cut off correct high-confidence tokens. Our methods preserve the conditional structure of NSR while adding selectivity.

\paragraph{Unlikelihood training~\cite{welleck2020unlikelihood}.} The unlikelihood gradient is:
\begin{equation}
-\frac{\partial \mathcal{L}_{\text{UL}}}{\partial z_v} \propto
\begin{cases}
-\pi_v & \text{if } v = y_t,\\
\frac{\pi_{y_t}}{1 - \pi_{y_t}}\pi_v & \text{if } v \neq y_t.
\end{cases}
\end{equation}
The \(\frac{1}{1-\pi_{y_t}}\)  factor adds updates for confident tokens, which can waste pretrained knowledge. NSR's \((1-\pi_v)\)  damping factor avoids this, and our CW-NSR preserves this protection while adding sample-level selectivity.

\subsection{Extension to PPO and GRPO}

Our methods compose naturally with clipped policy-gradient objectives. For PPO~\cite{schulman2017ppo}:
\begin{equation}
\mathcal{L}_{\text{A-PPO}}(\theta; t) = -\frac{1}{T}\sum_{t'=1}^{T} \min\!\left(
\frac{\pi_\theta}{\pi_{\text{old}}} w(t,r) A_{t'},\;
\text{clip}\!\left(\frac{\pi_\theta}{\pi_{\text{old}}}, 1-\epsilon, 1+\epsilon\right) w(t,r) A_{t'}
\right),
\end{equation}
where \(w(t,r)\)  includes both the adaptive schedule and (optionally) the confidence weight. Clipping limits the update magnitude but preserves the gradient direction, so our analysis remains qualitatively valid.

\section*{NeurIPS Paper Checklist}

The checklist is designed to encourage best practices for responsible machine learning research, addressing issues of reproducibility, transparency, research ethics, and societal impact. Do not remove the checklist: {\bf The papers not including the checklist will be desk rejected.} The checklist should follow the references and follow the (optional) supplemental material.  The checklist does NOT count towards the page
limit. 

Please read the checklist guidelines carefully for information on how to answer these questions. For each question in the checklist:
\begin{itemize}
    \item You should answer \answerYes{}, \answerNo{}, or \answerNA{}.
    \item \answerNA{} means either that the question is Not Applicable for that particular paper or the relevant information is Not Available.
    \item Please provide a short (1--2 sentence) justification right after your answer (even for \answerNA). 
\end{itemize}

{\bf The checklist answers are an integral part of your paper submission.} They are visible to the reviewers, area chairs, senior area chairs, and ethics reviewers. You will also be asked to include it (after eventual revisions) with the final version of your paper, and its final version will be published with the paper.

The reviewers of your paper will be asked to use the checklist as one of the factors in their evaluation. While \answerYes{} is generally preferable to \answerNo{}, it is perfectly acceptable to answer \answerNo{} provided a proper justification is given (e.g., error bars are not reported because it would be too computationally expensive'' or ``we were unable to find the license for the dataset we used''). In general, answering \answerNo{} or \answerNA{} is not grounds for rejection. While the questions are phrased in a binary way, we acknowledge that the true answer is often more nuanced, so please just use your best judgment and write a justification to elaborate. All supporting evidence can appear either in the main paper or the supplemental material, provided in appendix. If you answer \answerYes{} to a question, in the justification please point to the section(s) where related material for the question can be found.

IMPORTANT, please:
\begin{itemize}
    \item {\bf Delete this instruction block, but keep the section heading ``NeurIPS Paper Checklist"},
    \item  {\bf Keep the checklist subsection headings, questions/answers and guidelines below.}
    \item {\bf Do not modify the questions and only use the provided macros for your answers}.
\end{itemize}


\begin{enumerate}

\item {\bf Claims}
    \item[] Question: Do the main claims made in the abstract and introduction accurately reflect the paper's contributions and scope?
    \item[] Answer: \answerYes{} 
    \item[] Justification: The abstract and introduction clearly state the two proposed methods (A-NSR and CW-NSR), their motivation, and their evaluation on MATH, AIME 2025, and AMC23 with Qwen2.5-Math-1.5B. The scope is appropriately bounded to mathematical reasoning with verifiable rewards.

\item {\bf Limitations}
    \item[] Question: Does the paper discuss the limitations of the work performed by the authors?
    \item[] Answer: \answerYes{}{} 
    \item[] Justification: Section ~\ref{sec:limitations} explicitly discusses limitations, including: (1) potential instability in extended training runs, (2) inapplicability of CW-NSR in dense reward environments without reformulation, (3) current reliance on sequence-level rather than token-level confidence, and (4) restriction to mathematical reasoning tasks.

\item {\bf Theory assumptions and proofs}
    \item[] Question: For each theoretical result, does the paper provide the full set of assumptions and a complete (and correct) proof?
    \item[] Answer: \answerYes{} 
    \item[] Justification: The paper includes Propositions ~\ref{prop:convergence}- ~\ref{prop:variance} with proofs. Full derivations appear in Appendices ~\ref{app:gradients}, ~\ref{app:schedules}, and ~\ref{sec:theory}. The detached-weight assumption used in CW-NSR gradient derivation (Appendix ~\ref{ss:grad}) is clearly stated. Proposition ~\ref{prop:variance}'s variance bound is stated with its dependency on policy parameters made explicit.

    \item {\bf Experimental result reproducibility}
    \item[] Question: Does the paper fully disclose all the information needed to reproduce the main experimental results of the paper to the extent that it affects the main claims and/or conclusions of the paper (regardless of whether the code and data are provided or not)?
    \item[] Answer: \answerYes{} 
    \item[] Justification: All information needed to reproduce the main experimental results is disclosed. Training hyperparameters are fully specified in section ~\ref{sec:experiments} and summarized in Table ~\ref{tab:hyperparams} (Appendix ~\ref{app:implementation}.3), including batch size, learning rate, clip ratio, training temperature, and all method-specific hyperparameters. Prompt templates are provided in Appendix ~\ref{app:implementation}.2. The base model (Qwen2.5-Math-1.5B) and training dataset (MATH) are publicly available, and our implementation is released at the links provided in the footnote on page 1.

\item {\bf Open access to data and code}
    \item[] Question: Does the paper provide open access to the data and code, with sufficient instructions to faithfully reproduce the main experimental results, as described in supplemental material?
    \item[] Answer: \answerYes{} 
    \item[] Justification: Code for both proposed methods is publicly available at the repositories linked in the footnote on page 1. The trained model checkpoint is accessible via the HuggingFace link provided. The training data (MATH dataset, Hendrycks et al. 2021) is publicly available.

\item {\bf Experimental setting/details}
    \item[] Question: Does the paper specify all the training and test details (e.g., data splits, hyperparameters, how they were chosen, type of optimizer) necessary to understand the results?
    \item[] Answer: \answerYes{} 
    \item[] Justification: Section ~\ref{sec:experiments}.1 and Appendix ~\ref{app:implementation}.3 provide batch size, mini-batch size, learning rate, clip ratio, entropy bonus, training temperature, number of epochs, number of GPUs, evaluation temperature, top-p, and per-method hyperparameters (\(\beta_{\max}\), \(\beta_{\min}\), K, \(\lambda_{\min}\), \(\lambda_{\max}\), \(\alpha\), \(\epsilon\)). The evaluation metric (unbiased Pass@k estimator from Chen et al. 2021) is also stated with the formula.

\item {\bf Experiment statistical significance}
    \item[] Question: Does the paper report error bars suitably and correctly defined or other appropriate information about the statistical significance of the experiments?
    \item[] Answer: \answerYes{} 
    \item[] Justification: Our results are based on 256 random samplings and calculated with an unbiased estimator.

\item {\bf Experiments compute resources}
    \item[] Question: For each experiment, does the paper provide sufficient information on the computer resources (type of compute workers, memory, time of execution) needed to reproduce the experiments?
    \item[] Answer: \answerYes{} 
    \item[] Justification: We have provided the compute resources used in this expeiment in Section ~\ref{sec:experiments}
    
\item {\bf Code of ethics}
    \item[] Question: Does the research conducted in the paper conform, in every respect, with the NeurIPS Code of Ethics \url{https://neurips.cc/public/EthicsGuidelines}?
    \item[] Answer: \answerYes{} 
    \item[] Justification: This paper focuses on mathematical reasoning for LLMs using public benchmarks and poses no obvious ethical risks. No human subjects, sensitive data, or dual-use concerns are present.

\item {\bf Broader impacts}
    \item[] Question: Does the paper discuss both potential positive societal impacts and negative societal impacts of the work performed?
    \item[] Answer: \answerNA{} 
    \item[] Justification: The paper is foundational research on RL training objectives for mathematical reasoning. There is no direct path to harmful applications.
    
\item {\bf Safeguards}
    \item[] Question: Does the paper describe safeguards that have been put in place for responsible release of data or models that have a high risk for misuse (e.g., pre-trained language models, image generators, or scraped datasets)?
    \item[] Answer: \answerNA{} 
    \item[] Justification: No new datasets, scraped data, or high-risk pre-trained models are released. The method is an algorithmic contribution on top of an existing public model (Qwen2.5-Math-1.5B). No safeguards are needed beyond standard responsible disclosure.

\item {\bf Licenses for existing assets}
    \item[] Question: Are the creators or original owners of assets (e.g., code, data, models), used in the paper, properly credited and are the license and terms of use explicitly mentioned and properly respected?
    \item[] Answer: \answerYes{} 
    \item[] Justification: All third-party assets are cited: Qwen2.5-Math-1.5B (Yang et al. 2025b), the MATH dataset (Hendrycks et al. 2021), AIME 2025, AMC23, and the TRL library. The Pass@k estimator is attributed to Chen et al. 2021. Formal license names (e.g. Apache 2.0 for Qwen models) are not stated and should be added.

\item {\bf New assets}
    \item[] Question: Are new assets introduced in the paper well documented and is the documentation provided alongside the assets?
    \item[] Answer: \answerYes{} 
    \item[] Justification: We released our code, introduced our trained model and provided used dataset.

\item {\bf Crowdsourcing and research with human subjects}
    \item[] Question: For crowdsourcing experiments and research with human subjects, does the paper include the full text of instructions given to participants and screenshots, if applicable, as well as details about compensation (if any)? 
    \item[] Answer: \answerNA{} 
    \item[] Justification: No crowdsourcing or human subjects were involved.

\item {\bf Institutional review board (IRB) approvals or equivalent for research with human subjects}
    \item[] Question: Does the paper describe potential risks incurred by study participants, whether such risks were disclosed to the subjects, and whether Institutional Review Board (IRB) approvals (or an equivalent approval/review based on the requirements of your country or institution) were obtained?
    \item[] Answer: \answerNA{} 
    \item[] Justification: No human subjects are involved. IRB approval is not applicable.

\item {\bf Declaration of LLM usage}
    \item[] Question: Does the paper describe the usage of LLMs if it is an important, original, or non-standard component of the core methods in this research? Note that if the LLM is used only for writing, editing, or formatting purposes and does \emph{not} impact the core methodology, scientific rigor, or originality of the research, declaration is not required.
    \item[] Answer: \answerYes{} 
    \item[] Justification: LLMs (specifically Qwen2.5-Math-1.5B) are the subject of training, not used as a tool to assist in writing or analysis. The paper is about modifying the RL training objective for LLMs.

\end{enumerate}

\end{document}